\documentclass[conference]{ieeeconf}    

\usepackage{hyperref}
\usepackage{graphicx}
\usepackage{amsmath}
\usepackage{amssymb}
\usepackage{algorithm}
\usepackage{algorithmic}
\usepackage[utf8]{inputenc}
\usepackage[english]{babel}

\newtheorem{lemma}{Lemma}

\IEEEoverridecommandlockouts
\title{\LARGE \bf
    Inferring the Optimal Policy using Markov Chain Monte Carlo
}

\author{
Brandon Trabucco$^{1}$, Albert Qu$^{1, 2}$, Simon Li$^{1, 3}$, Ganeshkumar Ashokavardhanan$^{1}$\\
University of California, Berkeley\\
Department of \{Computer Science$^{1}$, Neuroscience$^{2}$, Math$^{3}$\}\\
\texttt{\{btrabucco, albert\_qu, xuanlinli17, ganeshkumar\}@berkeley.edu}
}

\begin{document}

\maketitle

\begin{abstract}

This paper investigates methods for estimating the optimal stochastic control policy for a Markov Decision Process with unknown transition dynamics and an unknown reward function. This form of model-free reinforcement learning comprises many real world systems such as playing video games, simulated control tasks, and real robot locomotion. Existing methods for estimating the optimal stochastic control policy rely on high variance estimates of the policy descent. However, these methods are not guaranteed to find the optimal stochastic policy, and the high variance gradient estimates make convergence unstable. In order to resolve these problems, we propose a technique using Markov Chain Monte Carlo to generate samples from the posterior distribution of the parameters conditioned on being optimal. Our method provably converges to the globally optimal stochastic policy, and empirically similar variance compared to the policy gradient. 

\end{abstract}

\section{Introduction}

Policy gradient techniques using neural networks are currently widely used methods for learning to interact with model-free environments. In these environments, the transition dynamics and the reward function are unknown, and the agent receives a reward signal for taking an action when given a state. These agents are able to learn complex control policies directly from reward signals. Examples of these control policies in recent literature include human level performance on Atari games \cite{humanlevel}, simulated environments \cite{ddpg, trpo}, and robot locomotion \cite{visiomotor}. These methods optimize a stochastic policy using a high variance estimate of the local gradient under the current policy of the expected future reward with respect to the policy parameters.

Current successful policy gradient algorithms fall into two broad categories: \textbf{actor-only} and \textbf{actor-critic}. In the \textbf{actor-only} scheme, the samples of the reward from an episode are used to compute the expected future reward directly. The log probability of actions is increased by an amount proportional to the expected future reward of taking that action. In this way, on average, actions that achieve higher expected future reward will be taken with higher probability \cite{reinforce}. In contrast, the \textbf{actor-critic} scheme fits a model of the expected future reward to samples of the current reward from the environment. This reward model is used to calculate the policy gradient update stated previously. Using the critic estimate of the reward serves to reduce variance in the traditional policy gradient algorithm. \cite{gae}. 

The traditional policy gradient algorithm inherits the limitations of stochastic gradient descent, and is susceptible to local minima. This arises in practice when policies with different random seeds finding different sub-optimal solutions to a control task. This problem exists in part because the environment has not been fully explored, and the model has learned to exploit the known high reward trajectories to which it has already been exposed. Both \textbf{actor-only} and \textbf{actor-critic} versions of the policy gradient suffer from the trade-off of exploration and exploitation. Algorithms that guarantee optimality, and therefore guarantee appropriate exploration of the environment are desirable to the reinforcement learning community.

In this paper, we hope to address the tendency of model-free reinforcement learning to converge to the local minima of the expected future reward. In particular, we make the following contributions:

\begin{enumerate}
    
    \item{ A closed form expression for the posterior distribution of the policy parameters that satisfy an optimality condition using Bayesian inference. }

    \item{ An extension of the Metropolis-Hastings variant of Markov Chain Monte Carlo for sampling these policy parameters with high probability and comparable variance to the policy gradient. }
    
\end{enumerate}

\section{ Related Work }

This work is not the first to apply Markov Chain Monte Carlo sampling methods for neural networks. Radford M. Neal provides an analysis of Bayesian learning for neural networks in his 1996 thesis, wherein he experiments with Metropolis Hastings and Gibbs sampling. In Neal's formulation, the posterior distributions of the parameters $\theta$ given the dataset of inputs and labels $D$ is approximated using Markov Chain Monte Carlo \cite{neal1996}. 

More recent work on Bayesian learning for neural networks focuses more on making the Metropolis Hastings update more efficient using mini-batches. Li, D. et al. prove that mini-batch Metropolis Hastings produces samples from the true posterior raised to a known temperature \cite{li2017}. Other work by Chen, H. et al. demonstrates that using mini-batches achieves several order-of-magnitude speedups from original Metropolis Hastings \cite{chen2016}. 

Using mini-batches to estimate the posterior probability also introduces variance. This is addressed by Baker, J. et al. where control variates are used to modify the dynamics of the Metropolis Hastings to reduce variance \cite{baker2017}. This is further explored by Gong, W. et al. where a meta-learning algorithm is proposed that efficiently optimizes for these control variates \cite{gong2018}. 

Our method differs from previous work in that we employ mini-batch Metropolis Hastings to solve for the optimal policy under a Markov Decision Process, where the typical "dataset" used to calculate the posterior in previous work is replaced with a reward signal. To the best of our knowledge, we are the first to apply Markov Chain Monte Carlo to solve Markov Decision Processes.

\section{Methodology}

\subsection{Theory}

Let $\mathcal{O}$ be the event when the policy $\pi$ is optimal. We can derive the optimal policy using Bayesian inference and integrate out the parameters $\theta$. We make the assumption that the next action $s$ is conditionally independent of $\mathcal{O}$ given some parameters $\theta$.

\begin{equation}
    \pi ( a | s, \mathcal{O} ) = \int \pi ( a | s, \theta ) f( \theta | \mathcal{O} ) d\theta
\end{equation}

Let us  define  the  sampling  distribution density $f(\theta | \mathcal{O})$ using Bayes Rule and the prior $f(\theta)$

\begin{equation}
    f( \theta | \mathcal{O} )  = \frac{f( \theta ) Pr( \mathcal{O} | \theta )}{ \mathbb{E}_{\theta' \sim f(\theta')} [ Pr(\mathcal{O} | \theta') ]}
\end{equation}

The optimal policy achieves maximum expected reward, and so we can define the probability of being optimal $\mathcal{O}$ given the current policy parameters $\theta$ by the continuous extension of the softmax function (Boltzmann distribution \cite{Li6611}) and a normalizing constant $A$.

\begin{equation}
    Pr( \mathcal{O} | \theta ) = \lim_{T \rightarrow 0} \frac{e^{\frac{\hat r_{\theta}}{T}}A}{\int_{\theta'} e^{\frac{\hat r_{\theta}}{T}}d\theta'}  = \lim_{T \rightarrow 0} e^{\frac{\hat r_{\theta}}{T}} Z(T)^{-1}
\end{equation}

The likelihood of being optimal is proportional to the exponential of the expected utility $\hat r_{\theta}$ of policy $\pi$, divided by some temperature constant $T$.

\begin{equation}
    f( \theta | \mathcal{O} ) = \frac{f( \theta ) e^{ \frac{\hat r_{\theta} }{ T} } Z(T)^{-1} }{\int_{\theta'} e^{ \frac{\hat r_{\theta '} }{ T} } Z(T)^{-1} f(\theta')d\theta'} \propto f( \theta ) e^{ \frac{\hat r_{\theta} }{ T} } 
\end{equation}

To evaluate $f(\theta | \mathcal{O})$, we need to first evaluate $\hat r_{\theta}$, which could be approximated by the following two methods:
\begin{enumerate}
    \item{ If we are sampling by executing actions from the distribution $\pi ( a | s, \theta )$ (on policy), we can compute the expected utility directly.
    $$ \hat r_{\theta} = \mathbb{E}_{s, a, s' \sim Pr(s) \pi ( a | s, \theta ) Pr(s' | s, a)} [ R (s, a, s') ] $$}
    
    \item{ If we are sampling from the environment using actions from the distribution $Pr(a | s)$ (off policy), we can compute the expected utility using importance sampling.
    $$ \hat r_{\theta} = \mathbb{E}_{s, a, s' \sim Pr(s) \pi ( a | s, \theta ) Pr(s' | s, a)} [ \frac{Pr(a | s)}{\pi ( a | s, \theta )} R (s, a, s') ] $$}
\end{enumerate}

Hitherto we could model the distribution $f( \theta | \mathcal{O} )$ as the stable distribution of an infinite Markov Chain, using Metropolis-Hastings algorithm (further specified in part $C$). Due to ergodicity, we could approximate the integral in equation (1) by drawing samples, in accordance with Metropolis Hastings Criterion, from the the Markov Chain after it runs close to stability. Since for small $T$, $$\exp(\hat r_{\theta^{*}} / T) \gg \exp(\hat r_{\theta^{'}}/T) \ \forall \ \theta^{'} \neq \theta^{*},$$ the probability of drawing $\theta^{*}$ from the near-stable Markov chain is salient compared to other possible $\theta^{'}$s. As \textit{Lemma 1} in $B$ suggests, the probability for the Markov Chain to transition to a suboptimal $\theta^{'}$ from $\theta^{*}$ as $T$ goes to $0$:
$$ \lim_{T \rightarrow 0} Pr(\theta^{'} \rightarrow \theta^{*}) \rightarrow 0.$$
Therefore, the $n$ sample $\theta$s we draw mostly likely would match $\theta^{*}$, 
and we could obtain the optimal policy.

$$\pi ( a | s, \mathcal{O} ) \approx \sum_{i=1}^n \frac{\pi(a|s, \theta)}{n} = \sum_{i=1}^n \frac{\pi(a|s, \theta^{*})}{n} = \pi(a|s, \theta^{*})$$


\subsection{Justification of correctness and optimality}

We prove that the on policy Metropolis-Hasting algorithm stated above correctly estimates the neural network parameter $\theta$ that maximizes the expected reward $ \hat r = \mathbb{E}_{s, a, s' \sim Pr(s) \pi ( a | s, \theta ) Pr(s' | s, a)} [ R (s, a, s') ] $ as $T\to 0$.

\begin{lemma}
Let $f:\theta \to \mathbb{R}$ be a continuous function with a maximum value at $\theta^*$, then 
$$ \frac{\phi(\theta)e^{\frac{f(\theta)}{T}}}{\int_{\theta'} \phi(\theta') e^{\frac{f(\theta')}{T}}d\theta'} $$
converges to 0 as $T\to 0$ whenever $f(\theta)$ is not maximal.
\end{lemma}

\begin{proof}
Note that $$ \frac{\phi(\theta)e^{\frac{f(\theta)}{T}}}{\int_{\theta'}\phi(\theta')e^{\frac{f(\theta')}{T}}d\theta'} =\frac{\phi(\theta)}{\int_{\theta'}\phi(\theta')e^{\frac{f(\theta')-f(\theta)}{T}}d\theta'}$$

Let $g(\theta')=f(\theta')-f(\theta)$. If $f(\theta)$ is not maximal, then $g(\theta^*)>0$. Since $g$ is continuous, there exists a neighborhood centered at $\theta^*$ such that $g(\theta')>\frac{1}{2}g(\theta^*)$ for all $\theta'$ in this neighborhood, so $\lim_{T\to 0}  \int_{\theta'}\phi(\theta')e^{\frac{f(\theta')-f(\theta)}{T}}d\theta'=\infty$, and the result follows.
\end{proof}

\bigskip
Note that at each $T$, the set of all $\theta$ is a discrete time Markov Chain with the transition probability between $\theta_1$ and $\theta_2$ being $\min(\frac{ \phi (\theta_2) \exp{( \hat r_{\theta_2} / T)} } { \phi (\theta_1) \exp{(\hat r_{\theta_1} / T)} }, 1)$. The stationary distribution of each $\theta$ is proportional to $\pi (\theta) = \phi (\theta) \exp{(\hat r_{\theta}/T)}$. As $T\to 0$, if $\hat r_{\theta}$ is not the maximum, then according to Lemma 1, $f(\theta | \mathcal{O}) \to 0$. Therefore, when $T$ is small, the Metropolis Hasting algorithm either remains in the current state or chooses a near optimal state with high probability.

\subsection{Algorithm Pseudo Code}

\begin{algorithm}
    \caption{On Policy Metropolis Hastings}
    \label{alg:thealg}
    \begin{algorithmic} 
    
    \REQUIRE $n > 0, \Sigma \succ 0, T > 0, \epsilon \in (0, 1)$
    
    \STATE $\theta_{0} \sim \mathcal{N} (0, \Sigma)$
        
    \STATE $ \hat r_{0} = \mathbb{E}_{s, a, s' \sim Pr(s) \pi(a | s, \theta_{0}) Pr(s' | a, s)} [ R(s, a, s') ] $
    
    \FOR{ $i \in \{1, \hdots, n\}$  }
    
        \STATE $\theta' \sim \mathcal{N} (\theta_{i - 1}, \Sigma)$
        
        \STATE $ \hat r' = \mathbb{E}_{s, a, s' \sim Pr(s) \pi(a | s, \theta') Pr(s' | a, s)} [ R(s, a, s') ] $
        
        \IF{ $\hat r' > \hat r_{i - 1}$}
        
             \STATE $ \theta_{i}, \hat r_{i} \leftarrow \theta', \hat r' $
             
        \ELSE
        
            \STATE $\alpha \leftarrow \frac{ \phi (\theta') \exp{( \hat r' / T)} } { \phi (\theta_{i - 1}) \exp{(\hat r_{i - 1} / T)} } $
            
            \STATE $ p \sim \mathcal{U} (0, 1) $
            
            \IF{ $\alpha \geq p$ }
            
                \STATE $ \theta_{i}, \hat r_{i} \leftarrow \theta', \hat r' $
            
            \ELSE
            
                \STATE $\theta_{i}, \hat r_{i} \leftarrow \theta_{i - 1}, \hat r_{i - 1}$
                
            \ENDIF
            
        \ENDIF
        
        \STATE $T \leftarrow \epsilon T$
    
    \ENDFOR
    
    \RETURN $ (\theta_{0}, \hdots, \theta_{n}) $
        
    \end{algorithmic}
\end{algorithm}

In this section we present the on policy derivation of the Metropolis-Hastings algorithm (Algorithm~\ref{alg:thealg}), which estimates the optimal policy parameters sampled from $f( \theta | \mathcal{O} )$. Averaging the policy resulting from many of these samples $\pi ( a | s, \theta )$ estimates the optimal policy $\pi ( a | s, \mathcal{O} )$.

Since the probability of transitioning away from the optimal parameters $\lim_{T \rightarrow 0} Pr(\theta^{*}, \theta^{'}) \rightarrow 0$ decays to zero, the optimal policy $\pi ( a | s, \mathcal{O} ) = \pi(a|s, \theta^{*})$ is simply the policy evaluated at the optimal parameter values.



\section{Experiments}

In this section, we design experiments to test Algorithm~\ref{alg:thealg}. We hope to test the algorithm in three domains: a randomly generated MDP toy problem, the OpenAI gym \texttt{CartPole-v0} environment, and in the Flow Traffic simulator \cite{wu2017}. In this section, the \textbf{temperature} refers to the variable $T$ in Algorithm~\ref{alg:thealg}, and the \textbf{cooling rate} refers to $\epsilon$ in Algorithm~\ref{alg:thealg}. The number of \textbf{iterations} refers to $n$, the number of times to simulate the Markov chain. Furthermore, when estimating the expected reward and expected future reward, the \textbf{batch size} refers to the number of samples to average. Finally, the \textbf{buffer size} hyper parameter is specific to the Gym experiment, which refers to the number of samples to collect from the environment, before random choosing exactly \textbf{batch size} of those samples. This hyper parameter serves to decorrelate the reward samples.

\subsection{Randomly Generated MDP}

The first experiment to test Algorithm~\ref{alg:thealg} is with a randomly generates Markov Decision Process. In this environment, the number of states is a hyper parameter $|S|$ and the number of actions is a hyper parameter $|A|$. The transition dynamics $Pr(s' | s, a) \in \mathcal{R}^{|S| \times |S| \times |A|}$ are randomly generated, and normalized to be a probability distribution. Finally, the reward function $R(s, a, s') \in \mathcal{R}^{|S| \times |A| \times |S|}$ is randomly generated to have positive and negative values.

At every iteration, the agent is given a reward according to the reward function. The task is to estimate the policy $\pi (a | s, \theta)$ that maximizes the expected reward. This task serves as a proof of concept toy example to verify the prospects of Algorithm~\ref{alg:thealg}. We plot the expected reward of the current policy for both the traditional policy gradient algorithm, and Algorithm~\ref{alg:thealg}. We also plot the difference in running time of each technique.

The policy gradient has a learning rate of 0.1, and a batch size of 512. Algorithm~\ref{alg:thealg} is instantiated with an initial temperature of 1.0, a cooling rate of 0.999, and a batch size of 512. In both cases the policy is optimized / sampled for 10000 iterations. The experiment is repeated 10 times, and the mean and standard deviation across trials is taken.

\subsection{Open AI Gym Tasks}

The second experiment to test Algorithm~\ref{alg:thealg} is using the OpenAI gym \texttt{CartPole-v0} environment. In this environment, the state space is continuous valued, with four slots $s \in \mathcal{R}^{4}$, organized by Table~\ref{fig:cartpole_table}.

\begin{figure}[h]
    \centering
    \begin{tabular}{r|l|r|r}
    Num & Observation & Min & Max \\
    \hline
    0 & Cart Position & -2.4 & 2.4\\
    0 & Cart Velocity & -Inf & Inf\\
    0 & Pole Angle & $-41.8^{\circ}$ & $41.8^{\circ}$ \\
    0 & Pole Velocity At Tip & -Inf & Inf\\
\end{tabular}
    \caption{Example of the \texttt{CartPole-v0} environment in OpenAI gym.}
    \label{fig:cartpole_table}
\end{figure}

At every non terminal iteration, the agent is given a reward of 1.0. The action space for \texttt{CartPole-v0} is discrete, with two possible values $a \in \{0, 1\}$: force applied left, and force applied right. The task is to estimate the policy $\pi (a | s, \theta)$ that maximizes the expected future reward. We plot the expected future reward of the current policy for both the traditional policy gradient algorithm, and Algorithm~\ref{alg:thealg}. We also plot the difference in running time of each technique.

The policy gradient has a learning rate of 0.1, a buffer size of 10000, and a batch size of 512. Algorithm~\ref{alg:thealg} is instantiated with an initial temperature of 1.0, a cooling rate of 0.9, a buffer size of 10000, and a batch size of 512. In both cases the policy is optimized / sampled for 1000 iterations. The experiment is repeated 10 times, and the mean and standard deviation across trials is taken. 

\begin{figure}[h]
    \centering
    \includegraphics[width=0.6\linewidth]{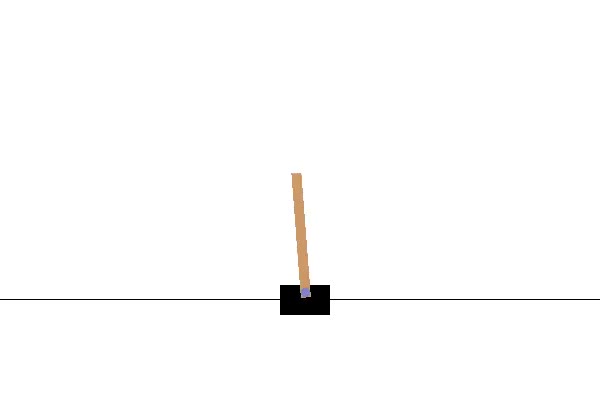}
    \caption{The \texttt{CartPole-v0} environment in OpenAI gym.}
    \label{fig:cartpole}
\end{figure}




\section{Analysis}

\subsection{Randomly Generated MDP}

\begin{figure}[h]
    \centering
    \includegraphics[width=0.7\linewidth]{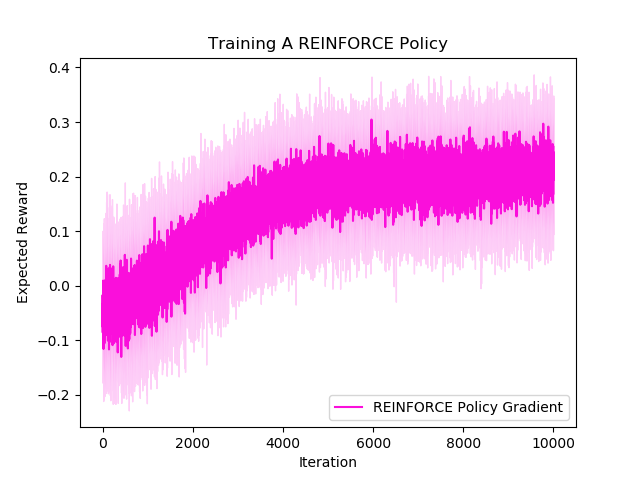}
    \caption{Training a stochastic policy using the reinforce policy gradient algorithm. The expected reward attained by the policy is plotted against the training iteration.}
    \label{fig:random_mdp_rl}
\end{figure}

\begin{figure}[h]
    \centering
    \includegraphics[width=0.7\linewidth]{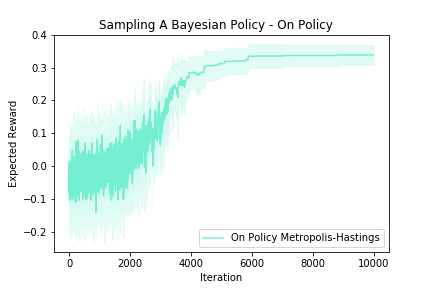}
    \caption{Sampling the parameters for a stochastic policy using the On Policy Metropolis Algorithm. The expected reward attained by the policy is plotted against the training iteration.}
    \label{fig:random_mdp_mh}
\end{figure}

As shown in Fig. 4, the On Policy Metropolis Hastings algorithm starts out with high temperature $T$, since acceptance probability for $(\hat r' - \hat r) < 0$:
$$\alpha = \frac{ \phi (\theta') e^{ (\hat r' - \hat r)/ T} } { \phi (\theta)}$$
 which is higher with larger $T$, facilitating transitions and resulting in the high variance. As the number of iterations increments, the On Policy MH algorithm explores different $\theta$s and transitions to policies with higher reward $r_{\theta\{MH\}}$. The variance starts to decrease at around $3000$ iterations and at around $4000$ iterations starts to approach optimality, with standard deviation $\sigma_{r_{\theta\{MH\}}}$ bounded within $0.02$. On the other hand, the Reinforcement learning algorithm yields policies with policies with similar optimal reward values, yet the standard deviation $\sigma_{r_{\theta\{RL\}}}$ remains high even after $4000$ iterations.
 
\begin{figure}[h]
    \centering
    \includegraphics[width=0.6\linewidth]{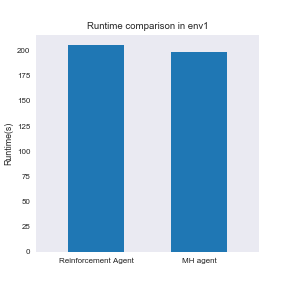}
    \caption{Average runtime difference between the policy gradient algorithm, and Algorithm~\ref{alg:thealg}. The policy gradient requires more time to compute an update that Algorithm~\ref{alg:thealg}.}
    \label{fig:env1_runtime}
\end{figure}

For these experiments, we calculate the total time required to complete the entire 10000 iterations. Our results demonstrate that the policy gradient requires more time finish the experiment (210 seconds), compared to Algorithm~\ref{alg:thealg} (190 seconds). This suggests that the On Policy Metropolis Hastings algorithm is more computationally efficient at finding the optimal policy that the policy gradient.



\subsection{Open AI Gym Tasks}

\begin{figure}[h]
    \centering
    \includegraphics[width=0.7\linewidth]{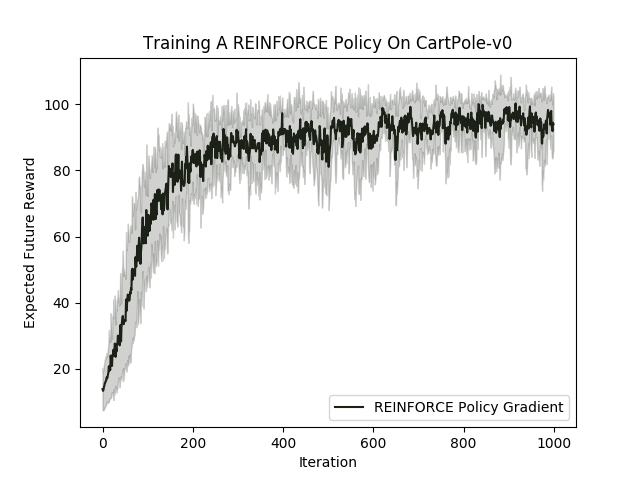}
    \caption{Training a stochastic policy using the reinforce policy gradient algorithm. The expected reward attained by the policy is plotted against the training iteration.}
    \label{fig:gym_rl}
\end{figure}

\begin{figure}[h]
    \centering
    \includegraphics[width=0.7\linewidth]{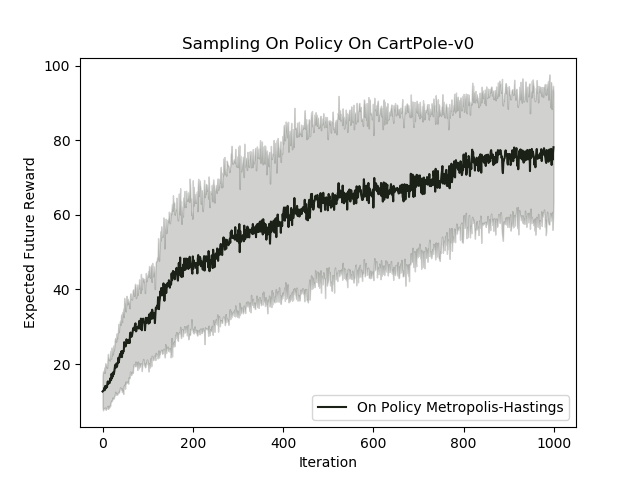}
    \caption{Sampling the parameters for a stochastic policy using the On Policy Metropolis Algorithm. The expected reward attained by the policy is plotted against the training iteration.}
    \label{fig:gym_mh}
\end{figure}

The On Policy Metropolis Hastings algorithm converges to a reward of about 80 around iteration 800 as is shown in the Fig. 6. Yet since temperature is an important hyper parameter that would help the convergence of $\hat r_{theta}$, the graph, which was a result of the MH algorithm without further hyper parameter tuning, the variance of the expected future reward is relatively high compared to RL agent. Nonetheless the reward convergence values of MH are comparable to that of RL agent and MH agent also has the runtime advantage (further described in \textit{VI.B.})

\begin{figure}[h]
    \centering
    \includegraphics[width=0.5\linewidth]{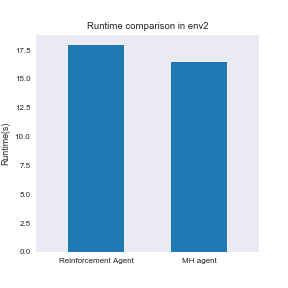}
    \caption{Average runtime difference between the policy gradient algorithm, and Algorithm~\ref{alg:thealg}. The policy gradient requires more time to compute an update that Algorithm~\ref{alg:thealg}.}
    \label{fig:env2_runtime}
\end{figure}

Similar to the last section, we calculate the total time required to complete the 100 training steps. Our results demonstrate that the policy gradient requires more time to finish the experiment (17.5 seconds), compared to Algorithm~\ref{alg:thealg} (16.0 seconds). This also suggests that the On Policy Metropolis Hastings algorithm is more computationally efficient than the policy gradient at updating the parameters.

\section{Discussion}

The On Policy MH Algorithm has many hyper parameters that can be tuned to converge to the optimal policy faster and with less variance. Tuning the hyper parameters is difficult and computationally intensive. Running the experiment multiple times takes several hours and thus it is difficult to find out better hyper parameters. Therefore, an efficient scheme for searching for and locating the optimal hyper parameters would increase the success of On Policy MH. 

The On Policy MH Algorithm depends on having an accurate estimate of the expected future reward, and this may not be available in certain environments, where the rewards samples have noise, or high variance. In memory constrained settings, storing a buffer of previous rewards may not be practical or scalable to larger environments. Therefore, methods for estimating the expected future reward, such as generalized advantage estimation \cite{gae} would increase the success of On Policy MH. 

We derived an off-policy version of the On Policy MH Algorithm in section III.A, but in practice this algorithm does not increase the expected future reward of the current policy. We believe that the introduction of importance sampling is biasing the estimate of the expected future reward, and preventing the algorithm from converging to an optimal policy.

\section{Conclusion}

From the analysis, we can conclude that On Policy Metropolis Hastings is able to converge a stochastic control policy that achieves comparable expected future reward to traditional model-free reinforcement learning. The algorithm does so with comparable variance to gradient based learning, while requiring less computation per update. With appropriately chosen hyper parameters, which remains a challenge by itself, we have proven our algorithm converges to the optimal stochastic policy. Therefore, our proposed solution is a viable alternative to traditional model-free reinforcement learning, which is only locally-optimal at best.

\bibliography{references}
\bibliographystyle{ieeetr}

\end{document}